\documentclass[letterpaper, 10 pt, conference]{ieeeconf}
\IEEEoverridecommandlockouts
\usepackage{xurl}
\usepackage{float}
\usepackage{xcolor}
\makeatletter\let\NAT@parse\undefined\makeatother
\usepackage[pdfa,colorlinks,bookmarksopen,bookmarksnumbered,allcolors=blue!90!black,urlcolor=blue]{hyperref}

\usepackage{array}
\usepackage{graphicx}
\usepackage{svg}
\usepackage{tabularx}
\usepackage{multirow}
\usepackage{booktabs}
\usepackage{colortbl}
\usepackage{subcaption}
\usepackage{amsmath}
\usepackage{amsfonts}
\usepackage{amssymb}
\usepackage{mathtools}
\usepackage{bm}

\usepackage[noadjust,space]{cite}

\usepackage[nameinlink,capitalise]{cleveref}
\crefname{line}{line}{lines}
\crefname{figure}{Fig.}{Figs.}
\Crefname{figure}{Fig.}{Figs.}
\crefname{equation}{Eq.}{Eqs.}
\Crefname{equation}{Eq.}{Eqs.}
\crefname{section}{Sec.}{Secs.}
\Crefname{section}{Sec.}{Secs.}
\crefname{definition}{Def.}{Defs.}
\Crefname{definition}{Def.}{Defs.}
\crefname{algorithm}{Alg.}{Algs.}
\Crefname{algorithm}{Alg.}{Algs.}
\crefname{assumption}{Asm.}{Asms.}
\Crefname{assumption}{Asm.}{Asms.}
\crefname{subassumption}{Asm.}{Asms.}
\Crefname{subassumption}{Asm.}{Asms.}
\Crefname{problem}{Problem}{Problems}
\crefname{problem}{Problem}{Problems}
\crefname{remark}{Rem.}{Rems.}
\Crefname{remark}{Rem.}{Rems.}
\crefname{corollary}{Cor.}{Cors.}
\Crefname{corollary}{Cor.}{Cors.}
\crefname{theorem}{Thm.}{Thms.}
\Crefname{theorem}{Thm.}{Thms.}
\crefname{lemma}{Lem.}{Lems.}
\Crefname{lemma}{Lem.}{Lems.}
\crefname{proposition}{Prop.}{Props.}
\Crefname{proposition}{Prop.}{Props.}

\usepackage{amsthm}
\newtheorem{theorem}{Theorem}
\newtheorem{proposition}{Proposition}
\newtheorem{lemma}{Lemma}

\newtheorem{assumption}{Assumption}

\usepackage[ruled,vlined,linesnumbered]{algorithm2e}
\crefname{algocf}{Alg.}{Algs.}
\Crefname{algocf}{Alg.}{Algs.}

\SetCommentSty{mycommfont}
\SetKwComment{Comment}{\(\triangleright\)\ }{}
\let\oldnl\nl
\newcommand{\nonl}{\renewcommand{\nl}{\let\nl\oldnl}}

\usepackage{siunitx}
\usepackage{pbalance}
\usepackage[export]{adjustbox}
\usepackage{wrapfig}
\usepackage[acronym]{glossaries}

\usepackage[
  activate   = {true},
  protrusion = false,
  expansion  = true,
  kerning    = true,
  spacing    = true,
  tracking   = false,
  auto       = true,
  selected   = true,
  factor     = 1000,
  stretch    = 10,
  shrink     = 10,
]{microtype}

\makeatletter
\newcommand\notsotiny{\@setfontsize\notsotiny\@vipt\@viipt}
\newcommand{\removelatexerror}{\let\@latex@error\@gobble}
\makeatother

\newcommand{\saiAns}[1]{\textcolor{blue}{\textbf{[For Sai]}~#1}}

\renewcommand{\saiAns}[1]{}

\DeclareMathOperator*{\argmin}{\arg\!\min}

\renewcommand{\H}[2][]{
\ifthenelse {\equal{#1}{}}
{\mathbb{H}\left[#2\right]}
{\mathbb{H}_{#1}\left[#2\right]}}

\newcommand{\E}[2][]{
\ifthenelse {\equal{#1}{}}
{\mathbb{E}\left[#2\right]}
{\mathbb{E}_{#1}\left[#2\right]}}

\def\1{\bm{1}}

\def\vq{{\bm{q}}}

\def\vu{{\bm{u}}}
\def\vv{{\bm{v}}}

\def\mC{{\bm{C}}}

\def\mG{{\bm{G}}}

\def\mQ{{\bm{Q}}}

\def\mV{{\bm{V}}}

\DeclareMathAlphabet{\mathsfit}{\encodingdefault}{\sfdefault}{m}{sl}
\SetMathAlphabet{\mathsfit}{bold}{\encodingdefault}{\sfdefault}{bx}{n}

\def\gA{{\mathcal{A}}}

\def\gE{{\mathcal{E}}}

\def\gG{{\mathcal{G}}}
\def\gH{{\mathcal{H}}}

\def\gK{{\mathcal{K}}}

\def\gP{{\mathcal{P}}}

\def\gX{{\mathcal{X}}}

\def\sR{{\mathbb{R}}}

\makeatletter
\newcommand{\StatexIndent}[1][3]{%
  \setlength\@tempdima{\algorithmicindent}%
  \Statex\hskip\dimexpr#1\@tempdima\relax}
\makeatother

\newcommand{\citet}{\cite}
\newacronym{poe}{PoE}{product of experts}
\newacronym{kld}{KL divergence}{Kullback–Leibler Divergence}
\newacronym{kl}{KL}{Kullback–Leibler}
\newacronym{map}{MAP}{maximum a Posteriori}
\newacronym{rmp}{RMP}{Riemannian Motion Policies}
\newacronym{gmm}{GMM}{Gaussian Mixture Model}
\newacronym{mpc}{MPC}{Model Predictive Control}
\newacronym{gp}{GP}{Gaussian Process}
\newacronym{ot}{OT}{Optimal Transport}

\newacronym{dmdp}{DMDP}{Deterministic Markov Decision Process}
\newacronym{mdp}{MDP}{Markov Decision Process}
\newacronym{vi}{VI}{Value Iteration}
\newacronym{dof}{DoF}{degrees of freedom}
\newacronym{sdf}{SDF}{Signed Distance Field}
\newacronym{mpot}{MPOT}{Motion Planning via Optimal Transport}
\newacronym{chomp}{CHOMP}{Covariant Hamiltonian Optimization for Motion Planning}
\newacronym{gpmp}{GPMP}{Gaussian Process Motion Planning}
\newacronym{stomp}{STOMP}{Stochastic Trajectory Optimization for Motion Planning}
\newacronym{sgpmp}{SGPMP}{Stochastic Gaussian Process Motion Planning}
\newacronym{bps}{BPS}{Batch Polytope Search}

\newif\ifanonymous

\newif\ifextended
\extendedfalse

\title{\LARGE \bf
  Anytime Global Tensor Motion Planning
}
\ifanonymous
\author{Anonymous Authors
  \thanks{Anonymized}%
  \thanks{}%
  \thanks{}%
  \thanks{}%
}
\else
\author{Sai Coumar\(^{1}\), An T. Le\(^{2,3,*}\), and Zachary Kingston\(^{1,*}\)%
  \thanks{\(^{*}\)Equal contribution.}%
  \thanks{\(^{1}\)Department of Computer Science, Purdue University, USA.}%
  \thanks{\(^{2}\)College of Engineering and Computer Science, VinUniversity, Hanoi, Vietnam.}%
  \thanks{\(^{3}\)Intelligent Autonomous Systems Lab, TU Darmstadt, Germany.}%
}
\fi

\begin{document}

\maketitle
\thispagestyle{empty}
\pagestyle{empty}

\begin{abstract}
  Global Tensor Motion Planning (GTMP) solves motion planning with batched tensor operations over a layered multipartite graph.
  We generalize GTMP so that adjacent-layer edges are realized by any black-box local planner (e.g., linear interpolation, splines, sampling-based planning, trajectory optimization, or generative sampling).
  We provide two anytime policies on top of this generalization: Anytime GTMP with random restarts at a fixed budget, which covers every homotopy class almost surely, and AO-GTMP with informed expansion with growing budgets, which converges to the optimal cost.
  We prove that a single sampled graph covers every endpoint-fixed homotopy class admitting a \(\delta\)-clear representative of bounded length.
  We also prove that additional samples per layer reduce the per-layer miss probability exponentially, whereas stronger local planners reduce the required layer count only sublinearly.
  On manipulation benchmarks the method matches state-of-the-art performance, and on 2D navigation it returns batches of topologically diverse solutions, while the informed baselines concentrate on one or two classes.
\end{abstract}

\begin{figure*}[t]
  \centering
  \includegraphics[width=0.4965\textwidth]{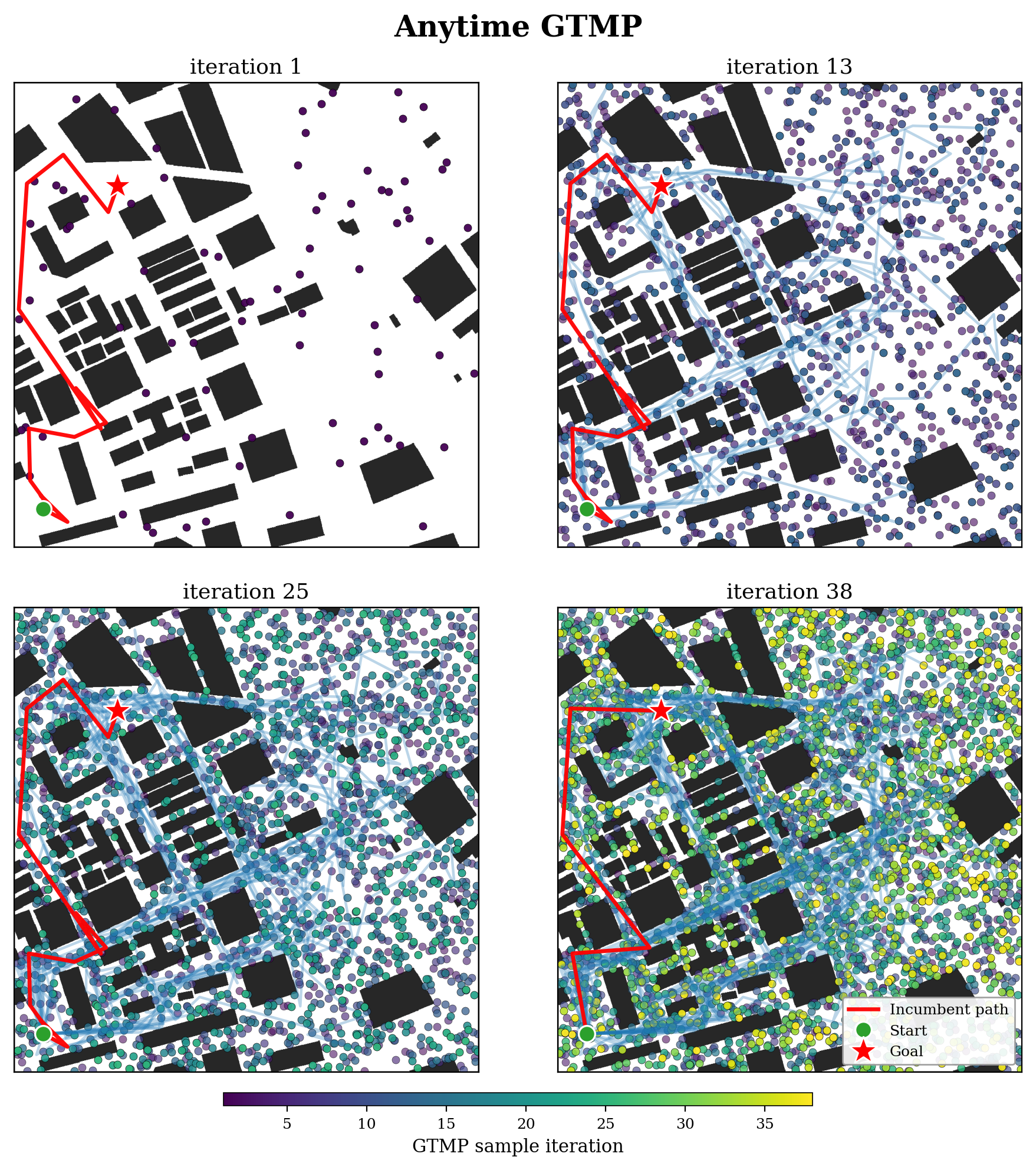}\hspace{0.006\textwidth}%
  \includegraphics[width=0.4965\textwidth]{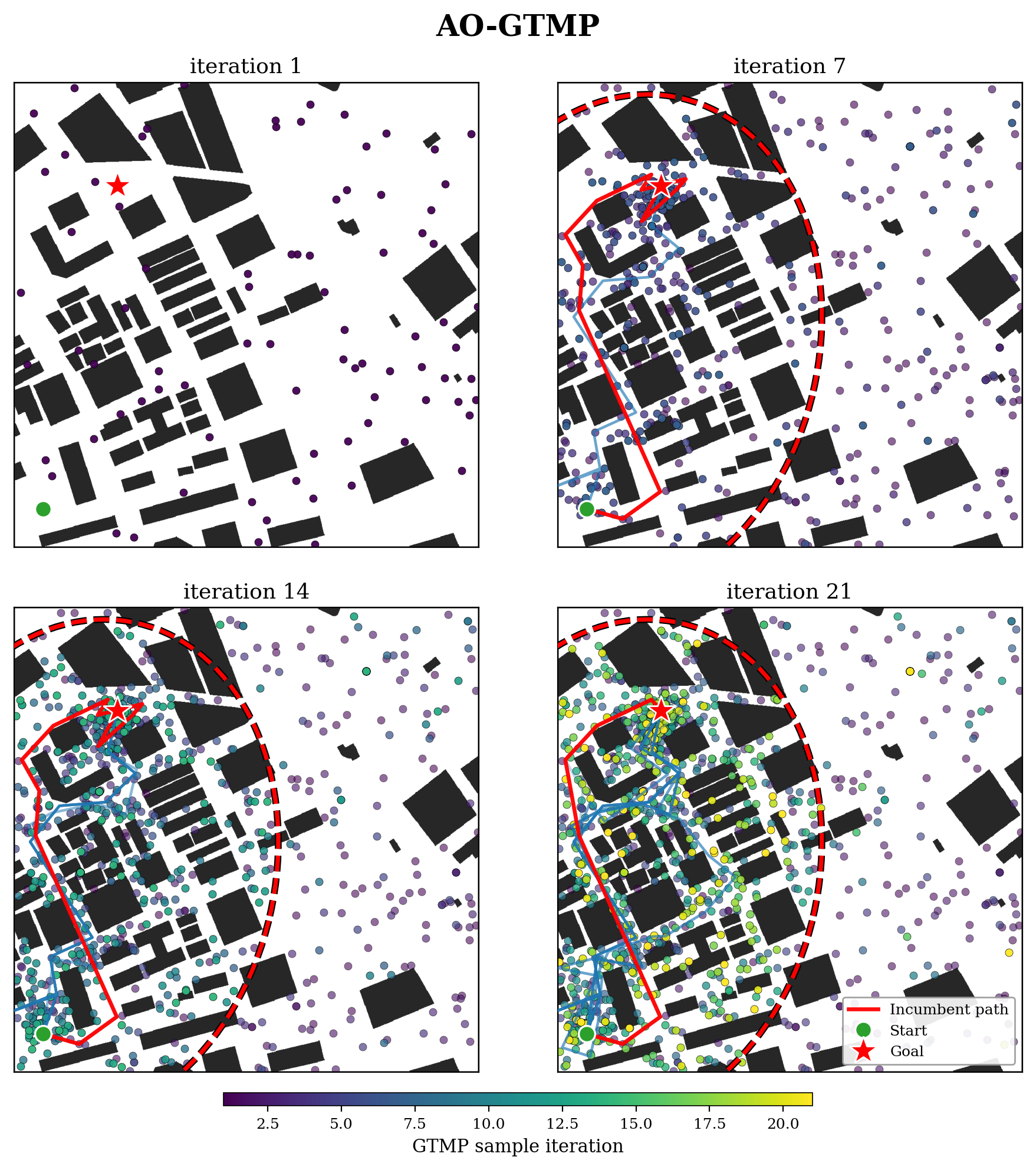}
  \caption{Sampling density of Anytime-GTMP (left) and AO-GTMP (right) on the Shanghai map.}
  \label{fig:cover}
\end{figure*}
\section{Introduction} \label{sec:intro}

Motion planning in cluttered, high-dimensional spaces requires finding the connected components of the configuration space and connecting configurations within them~\cite{orthey2023sampling}.
Classical sampling-based planners do both in one sequential search, which couples global exploration to local connection failure.
When several routes exist, downstream applications need \emph{topologically diverse} alternatives rather than one feasible path: navigation selects among routes around obstacles, manipulation exploits distinct approach paths, and task-level planning reasons over the alternatives.
Existing topological planners reason about homotopy classes through sequential roadmap construction~\cite{schmitzberger2002hppr,jaillet2008deformation,novak2023ctopprm}, without coverage guarantees.

Global Tensor Motion Planning (GTMP)~\cite{le2025global} recasts sampling-based planning as two batch-parallel operations, sampling the intermediate vertex layers and realizing all adjacent-layer edges, so many candidate paths are evaluated at once.
We generalize GTMP with black-box \emph{local planners} (LPs), admitting straight-line interpolation, sampling-based planners, trajectory optimizers, and generative samplers.
We prove that this method covers \emph{every} homotopy class admitting a \(\delta\)-clear representative of bounded length.
Our key insight is that a tube of clearance \(\delta\) around a reference path guarantees coverage: the layer samples explore globally, the local planner connects locally, and any path that stays inside the tube lies in the reference's homotopy class.
Moreover, random restarts at a fixed budget achieves almost-sure class coverage (Anytime-GTMP), and informed expansion with growing budgets gives almost-sure cost convergence (AO-GTMP).

We validate multi-class coverage on 2D navigation benchmarks, where Anytime-GTMP returns batches of topologically diverse solutions while the baselines return one path per query.
\Cref{fig:cover} contrasts the two modes: random restart explores new corridors, while asymptotic optimality concentrates samples in the shrinking informed set.
On MotionBenchMaker~\cite{chamzas2022-motion-bench-maker} with 6--8 degree-of-freedom (DoF) manipulators, our variants match state-of-the-art performance and attain the lowest mean path cost on 15 of 21 problems within 60 seconds.
\ifanonymous
    We release our code as open source at \url{https://anonymous.4open.science/r/anytime_gtmp-61B2}.
\else
    We release our code as open source at \url{https://github.com/commalab/anytime_gtmp.git}.
\fi
\section{Related Work} \label{sec:related_works}

Sampling-based planners randomly sample the configuration space and use a local planner to connect nearby configurations~\cite{orthey2023sampling}.
Narrow passages are difficult as the probability of exploration decays exponentially with dimension, which motivates biased sampling~\cite{hsu2003bridge,qureshi2021mpnet,carvalho2024motion}.
The choice of local planner also affects performance~\cite{amato2000choosing}: prior work has put other planners~\cite{nielsen2000two,hauser2011randomized} and learned functions~\cite{faust2018prm} in that role, but each choice is fixed inside the planner.
We extend GTMP with black-box local planners and prove graph-level coverage and convergence guarantees that hold for any of them.

To obtain almost-sure cost convergence, RRT*~\cite{karaman2011sampling} rewires the search tree while planners like BIT* and AIT*~\cite{gammell2020bit,strub2022adaptively} search batches of samples ordered by potential solution quality within the informed set~\cite{gammell2018informed}.
Unlike methods that densify the configuration-space approximation, the AO-x framework~\cite{hauser2016asymptotically,kleinbort2020refined} is a general recipe: any feasible planner can be made asymptotically optimal by iteratively tightening a cost bound in an augmented state-cost space, as AORRTC~\cite{wilson2025aorrtc} does with RRT-Connect.
AO-GTMP applies the same cost-bounding strategy to a layered DAG, replacing nearest-neighbor queries with a tensor value-iteration sweep.

Sampling-based planners admit natural parallelism~\cite{amato1999probabilistic}.
CPU planners exploit it by running independent instances~\cite{otte2013c}, growing trees in parallel~\cite{ichnowski2012parallel}, maintaining forests~\cite{Plaku2005,jacobs_scalable_method_2012}, and vectorizing collision checking~\cite{thomason2024motions}, and GPU planners such as GMT~\cite{ichter2017group}, pRRTC~\cite{huang2025prrtc}, and cuRobo~\cite{sundaralingam2023curobo} parallelize similarly at scale.
GTMP~\cite{le2025global} instead parallelizes through tensor operations over a layered multipartite graph, enabling batch-parallel edge realization.

When topologically distinct routes exist, downstream tasks benefit from a set of alternatives rather than one path.
The h-signature~\cite{bhattacharya2012topological} and persistent-homology descriptors~\cite{pokorny2016topological} classify paths by homotopy, and several planners use these invariants to construct diverse solutions: homotopy-preserving roadmaps~\cite{schmitzberger2002hppr}, path-deformation methods~\cite{jaillet2008deformation}, diverse-short-path algorithms~\cite{voss2015diverse,orthey2019motion}, convex-dissection approaches~\cite{liu2023convex}, CTopPRM~\cite{novak2023ctopprm}, and topology-driven trajectory optimization~\cite{degroot2024topology,osa2020multimodal}.
These methods build or search incrementally and give no coverage guarantee; the closest quantitative results are sample-complexity bounds for PRM connectivity via \(\varepsilon\)-nets~\cite{tsao2020sample}, which do not distinguish homotopy classes.
We guarantee that one graph covers every \(\delta\)-clear class at once, and Anytime-GTMP makes coverage almost-sure.
\section{Problem Definition} \label{sec:problem}

Let \(C\subseteq \mathbb{R}^d\) be the configuration space, with open free space \(C_{\mathrm{free}}\), obstacle region \(C_{\mathrm{coll}} = \mathbb{R}^d\setminus C_{\mathrm{free}}\), and start and goal configurations \(q_0,q_g\in C_{\mathrm{free}}\).
A continuous path is a map \(f:[0,1]\to C_{\mathrm{free}}\) with \(f(0)=q_0\) and \(f(1)=q_g\), and its image is \(\mathrm{Im}(f) = f([0,1])\).
A path \(f\) is \emph{\(\delta\)-clear} if
\(\operatorname{dist}(f(t),C_{\mathrm{coll}}) \ge \delta\) for all \(t\in[0,1]\)~\cite{karaman2011sampling}; equivalently, \(N_\delta(\mathrm{Im}(f))\subset C_{\mathrm{free}}\).
A homotopy class is \emph{\(\delta\)-clear} if it admits a \(\delta\)-clear representative.
A path is \emph{rectifiable} if it has finite Euclidean length. We write \(\mathrm{TV}(f)\) for that length.
Two paths \(f,g:[0,1]\to C_{\mathrm{free}}\) with the same endpoints are \emph{homotopic relative to endpoints},
written \(f\simeq g\), if there exists a continuous map
\(H:[0,1]\times[0,1]\to C_{\mathrm{free}}\) such that \(H(\cdot,0)=f\), \(H(\cdot,1)=g\),
\(H(0,s)=q_0\), and \(H(1,s)=q_g\) for all \(s\in[0,1]\).
The homotopy class of \(f\) is denoted by \([f]\).

\subsection{Layered Sampling Model}
Fix an integer \(M\ge 0\) and define the uniform grid on \([0,1]\)
\[
  t_m = \frac{m}{M+1}, \qquad m=0,1,\dots,M+1.
\]
For each intermediate layer \(m\in\{1,\dots,M\}\), let \(C_m\subseteq C\) be the \emph{sampling region} of that layer: a measurable set with nonempty interior and finite volume. Let \(\pi_m\) be a probability measure on \(C_m\cap C_{\mathrm{free}}\) with Lebesgue density at least \(a_m>0\) there.
Each of the \(N\) layer-\(m\) samples \(X_{m,1},\dots,X_{m,N}\) is drawn by an independent fair coin: with probability \(\tfrac12\) from \(\pi_m\), independently of everything else, and otherwise from any heuristic sampler, which may depend on the history.
The \(\pi_m\) draws are mutually independent within and across layers, so a layer misses a measurable set \(A\) with probability at most \((1-\tfrac12\pi_m(A))^N\) whatever the heuristic does.
Every result below uses only the \(\pi_m\) half of the mixture.
Define the layer vertex sets by:
\begin{gather*}
  V_0=\{q_0\}, \quad V_{M+1}=\{q_g\}, \quad V_m=\{X_{m,1},\dots,X_{m,N}\}.
\end{gather*}
The GTMP candidate graph contains all ordered pairs \((u,v)\) with \(u\in V_m\) and \(v\in V_{m+1}\) for some
\(m\in\{0,\dots,M\}\).
The graph is complete between adjacent layers and contains no intra-layer or skip-layer edges.

\subsection{Generalized GTMP Graph}
A \emph{local planner}~(LP), denoted \(\mathsf{LP}\), is a randomized procedure \(\mathsf{LP}(x,y;s,\ell)\) that, given \(x,y\in C_{\mathrm{free}}\), an effort budget \(s\in\mathbb{N}\), and a limit \(\ell>0\), either returns failure or returns a continuous collision-free path \(\Gamma:[0,1]\to C_{\mathrm{free}}\) with \(\Gamma(0)=x\), \(\Gamma(1)=y\), and image inside the \emph{query ellipsoid}
\[
  \gE(x,y;\ell) \;=\; \{\, z\in C_{\mathrm{free}} \mid \|z-x\|+\|z-y\| \le \ell \,\}.
\]
$\ell$ is computable from the query alone, is enforced by rejecting samples and edges outside \(\gE\), and is not restrictive: any path from \(x\) to \(y\) of length at most \(\ell\) already lies in \(\gE\), the segment \([x,y]\) included.
The corresponding graph is denoted by
\(G^{\mathsf{LP}}_{M,N,s}\).
An edge \((u,v)\) is present in \(G^{\mathsf{LP}}_{M,N,s}\) if and only if the call succeeds; the edge holds the returned local path \(\Gamma_{u,v}\) and its cost.

A \emph{chain} in \(G^{\mathsf{LP}}_{M,N,s}\) is a sequence of vertices: \[
  v_0=q_0,\ v_1,\ \dots,\ v_M,\ v_{M+1}=q_g,
\]
with \(v_m\in V_m\) and each adjacent pair \((v_m,v_{m+1})\) realized as an edge.
If the realized edges carry local paths \(\Gamma_0,\dots,\Gamma_M\), then their concatenation
\(
  P = \Gamma_0 * \cdots * \Gamma_M
\)
is a continuous collision-free path from \(q_0\) to \(q_g\).

Let \([f]\) be an endpoint-fixed homotopy class from \(q_0\) to \(q_g\).
We say that \(G^{\mathsf{LP}}_{M,N,s}\) \emph{covers} \([f]\) if it contains a chain whose
concatenated path \(P\) satisfies \( P \simeq f\).
For a finite family of classes
\(
  \mathcal{F} = \{[f^{(1)}],\dots,[f^{(K)}]\},
\)
we say that \(G^{\mathsf{LP}}_{M,N,s}\) \emph{covers \(\mathcal{F}\)} if it covers every class in the family, not
necessarily via disjoint chains.
Class membership of a returned path \(P\) is recorded by a label \(\kappa(P)\) that accumulates along the path, such as the
h-signature~\cite{bhattacharya2012topological} or a persistent-homology descriptor~\cite{pokorny2016topological}; we write \(\gK\) for the finite set of labels realized by chains of a given graph.
\begin{figure}[!t]
  \removelatexerror
  \begin{algorithm}[H]
    \caption{GTMP and our anytime variants}\label{algo:gtmp}
    \linespread{1.15}\small
    \DontPrintSemicolon
    \SetInd{0.5em}{0.5em}%
    \SetKwProg{func}{Function}{}{}%

    \func{\upshape\texttt{gtmp}\((\vq_0, \mG, s, \mQ, \mV_h)\)}{
      \KwData{Start \(\vq_0\), goals \(\mG \in \sR^{|\gG| \times d}\), local planner budget \(s\), (optional) samples \(\mQ\), value tensor \(\mV_h\)}
      \uIf{\(\mQ \equiv \emptyset\)}{
        Uniformly sample \(\mQ \in \sR^{M \times N \times d}\)\;
      }
      Compute cost matrices \(\mC_s, \mC_h, \mC_l\) via local planner \(\mathsf{LP}\) with budget \(s\) as in GTMP~\cite{le2025global}\;
      Init \(V_s \in \sR,\, \mV_g \in \sR^{|\gG|}\);\, \(\mV_h \leftarrow \mathbf{0}\) \upshape\textbf{if} \(\mV_h \equiv \emptyset\)
      \For{\(1 \leq k \leq M + 1\)}{
        \(\mV_h[M{-}1] \leftarrow \min (\mC_l + \mV_g)\);\, \(V_s \leftarrow \min (\mC_s + \mV_h[0])\)\;
        \lForEach{\(m < M{-}1\)}{\(\mV_h[m, u, \kappa_{uv}\chi] \leftarrow \min_v (\mC_h[m, u, v] + \mV_h[m{+}1, v, \chi])\)}
      }
      \(i \leftarrow \argmin (\mC_s + \mV_h^*[0])\);\, \(\gP = \{i\}\)\;
      \lFor{\(1 \leq m \leq M - 1\)}{
        \(i \leftarrow \argmin (\mC_h[m {-} 1, i] + \mV_h^*[m])\), append \(\mQ[m, i]\) to \(\gP\)
      }
      \(i \leftarrow \argmin (\mC_l[i] + \mV_g)\) and append \(\mG[i]\) to \(\gP\)\;
      \Return{Path \(\gP\),\, value tensor \(\mV_h\)}\;
    }{}

    \vspace{0.5em}
    \func{\upshape\texttt{anytime{-}gtmp}\((M, N, s)\)}{\label{algo:anytime}
      \(\gA \gets \emptyset\) \tcp*{class-indexed archive}
      \Repeat{\upshape\texttt{stop}}
      {
        \(\Pi_i \gets \)\upshape\texttt{gtmp}\((M, N, s)\) with \(\mV_h \in \sR^{M \times N \times |\gK|}\)\; \label{algo:anytime:gtmp}
        \ForEach{\(\chi \in \gK\) \upshape\textbf{with} \(\Pi_i[\chi] \not\equiv \emptyset\)}{
          \lIf{\(\chi \notin \gA\) \upshape\textbf{or} \(c(\Pi_i[\chi]) < c(\gA[\chi])\)}
          {
            \(\gA[\chi] \gets \)\upshape\texttt{simp}\((\Pi_i[\chi])\)
          }
        }
      }
      \Return{\(\gA\),\, \(\argmin_{\gP \in \gA} c(\gP)\)}\;
    }{}

    \vspace{0.5em}
    \func{\upshape\texttt{ao{-}gtmp}\((M_{\min}, M_{\max}, N_{\min}, N_{\max}, s)\)}{\label{algo:aogtmp}
      \(\gP_\textup{best} \gets \emptyset\);~
      \(c_\textup{min} \gets \infty\)\;
      \(M \gets M_{\min}\);~
      \(N \gets N_{\min}\)\;
      \(\mQ \gets \emptyset\);~
      \(\mV_h \gets \emptyset\)\; \label{algo:ao:init}
      \Repeat{\upshape\texttt{stop} \upshape\textbf{or} \((M \geq M_{\max}\) \upshape\textbf{and} \(N \geq N_{\max})\)}
      {
        \(\mQ \gets \)\upshape\texttt{expand{-}samples}\((\mQ, M, N, c_\textup{min})\)\; \label{algo:ao:expand}
        \(\gP_i, \mV_h \gets \)\texttt{simp}(\texttt{gtmp}\((M, N, s, \mQ, \mV_h))\)\; \label{algo:ao:gtmp}
        \If{\(\gP_i \not \equiv \emptyset\) \upshape\textbf{and} \(c(\gP_i) < c_\textup{min}\)}
        {
          \(\gP_\textup{best} \gets \gP_i\)\;
          \(c_\textup{min} \gets c(\gP_i)\)\;
        }
        \((M, N) \gets \)\upshape\texttt{grow}\((M, N)\)\; \label{algo:ao:grow}
      }
      \Return{\(\gP_\textup{best}\)}\;
    }{}

    \vspace{0.5em}
    \func{\upshape\texttt{expand{-}samples}\((\mQ, M, N, c_\textup{best})\)}{ \label{algo:ao:expand_fn}
      \For{layer \(m \in [M]\), \(1 \leq j \leq N - |\mQ[m]|\)}
      {
        \uIf{\(c_\textup{best} {=} \infty\) \textbf{or} \(\mQ[m{\pm}1] {=} \emptyset\) \textbf{or} \(\textup{Bernoulli}(\tfrac{1}{2})\)}
        {
          \(\vq \sim \mathcal{U}(\gX_f(c_\textup{best}))\) \tcp*{uniform half}
        }
        \lElse
        {
          \(\vq \gets \tfrac{1}{2}(\vu^\star {+} \vv^\star)\), \((\vu^\star, \vv^\star) {=} \argmin_{\mQ[m{-}1] \times \mQ[m{+}1]} \| \vu {-} \vv \|\) \label{algo:ao:pair}
        }
        \(\mQ[m] \stackrel{+}\gets \{\vq\}\)\;
      }
      \Return{\(\mQ\)}\;
    }{}
  \end{algorithm}%
  \vspace{-2em}
\end{figure}

\section{Method} \label{sec:method}
Two iteration policies run on the layered graph: \textbf{Anytime-GTMP} uses random restart with fixed budgets for homotopy-class diversity, and \textbf{AO-GTMP} grows budgets monotonically for asymptotic cost optimality.
The local planner may be any point-to-point planner; the original GTMP is the straight-line special case.

\subsection{Stagewise Graph Instantiation}
At planning stage \(\nu\in\mathbb{N}\), choose budgets \((M_\nu,N_\nu,s_\nu)\) and instantiate the layered graph from \cref{sec:problem}, yielding the realized graph \(G^{\mathsf{LP}}_\nu\).
Every call \(\mathsf{LP}(u,v;s_\nu,\ell)\) uses the edge-length limit \(\ell=\ell_{M_\nu}(r)\) of \cref{sec:homotopy_theory}, and on success assigns edge cost \(c_\nu(u,v) = \mathrm{cost}(\Gamma_{u,v})\); admissible costs are additive, at least path length, and Lipschitz in the endpoints of a straight-line edge (\cref{thm:aogtmp}), which includes path length and weighted length.
Per-stage cost is \(\mathcal{O}(M_\nu N_\nu^2)\) local planner calls, which dominates the search except for the \(|\gK|\) factor of the class-augmented sweep (\cref{sec:dp-search}).

\subsection{Graph Search on the Layered DAG}
\label{sec:dp-search}
Because GTMP connects only adjacent layers, the realized graph is a directed acyclic graph ordered by layer index, so a shortest feasible chain follows from value iteration at cost \(\mathcal{O}(M_\nu N_\nu^2)\).
Let \(\mathcal{V}_{\nu,m}\) denote the vertices in layer \(m\).
Define the terminal cost-to-go by \(J_{\nu,M_\nu+1}(q_g) = 0\), and for any vertex \(u\in\mathcal{V}_{\nu,m}\),
\(m=M_\nu,\dots,0\),
\[
  J_{\nu,m}(u) = \min_{v\in \mathcal{V}_{\nu,m+1}}
  c_\nu(u,v) + J_{\nu,m+1}(v).
\]
Any absent edge is treated as having infinite cost.
Taking \(\sigma_{\nu,m}(u)\) to be a minimizer above and applying it repeatedly from \(q_0\) recovers one minimum-cost feasible chain \(q_0 \to v_{\nu,1} \to \cdots \to v_{\nu,M_\nu} \to q_g\).
Concatenating the local planner paths yields a continuous configuration-space path:
\[
  P_\nu = \Gamma_{q_0,v_{\nu,1}} * \Gamma_{v_{\nu,1},v_{\nu,2}} * \cdots * \Gamma_{v_{\nu,M_\nu},q_g}.
\]
Anytime-GTMP augments the state with the class label, \(J_{\nu,m}(u,\chi)\) over \(\mathcal{V}_{\nu,m}\times\gK\), and sweeps the same recursion subject to \(\chi = \kappa(\Gamma_{u,v})\cdot\chi'\), which returns a minimum-cost chain in \emph{every} realized class for \(\Theta(M_\nu N_\nu^2|\gK|)\) work whenever the label accumulates along concatenation and tells the classes apart, as the h-signature does (\cref{prop:class-dp}).
Each sweep is a single batched tensor contraction.

\subsection{Two Iteration Policies}
\label{sec:two-modes}

The two policies differ only in the budget schedule: Anytime-GTMP holds \((M,N,s)\) fixed and targets class diversity, while AO-GTMP holds \(s\) fixed, grows \((M_\nu,N_\nu)\) monotonically to infinity, and targets cost optimality.
Both proceed in two phases.
\emph{Phase 1 (feasibility).} Run GTMP from \((M_1,N_1,s_1)\) until one feasible chain is found, yielding \(\sigma_0\) and the initial bound \(c_0\gets\mathrm{cost}(\sigma_0)\); with a probabilistically complete local planner this phase inherits the probabilistic completeness of GTMP.
\emph{Phase 2 (anytime).} Iterate under the chosen policy.

\Cref{algo:gtmp} gives both policies and the shared \upshape\texttt{gtmp}\(()\) subroutine, which extends canonical GTMP with the local planner budget \(s\) and optional warm-start inputs \(\mQ\) and \(\mV_h\).
The first iteration returning a feasible path is Phase~1: it seeds the archive (Anytime-GTMP) or the cost bound \(c_\textup{min}\) (AO-GTMP).
Both policies cost \(\Theta(M_\nu N_\nu^2)\) per stage, the \(2N_\nu+(M_\nu{-}1)N_\nu^2\) candidate edges being dominated by inter-layer pairs.

\subsection{Anytime-GTMP}
\label{sec:gtmp-rr}
Fix \((M, N, s)\) throughout.
Each stage \(\nu\) draws a fresh layered sample set, realizes edges with \(\mathsf{LP}_s\), runs the class-augmented DP
search, and inserts every solution into a class-indexed archive \(\mathcal{A}_\nu\) that holds the
lowest-cost path per homotopy class, updated by:
\[
  \mathcal{A}_\nu[\kappa] \gets
  \argmin_{\pi \in \mathcal{A}_{\nu-1}[\kappa] \cup \{\pi' \in \Pi_\nu \mid \kappa(\pi') = \kappa\}}
  \mathrm{cost}(\pi),
\]
with \(\mathcal{A}_0 \equiv \emptyset\), and the anytime output is the lowest-cost path across all classes.
Any positive per-stage coverage probability yields almost-sure eventual class coverage (\cref{thm:anytime-rr}).

\subsection{AO-GTMP}
\label{sec:aogtmp}
Fix \(s\) throughout; grow \((M_\nu, N_\nu)\) monotonically.
This instantiates the AO-x meta-algorithm~\cite{hauser2016asymptotically,kleinbort2020refined} on the
layered DAG.
Let \(c_\nu\) be the best cost found through stage \(\nu\) and let
\(\mathcal{X}_f(c) = \{ x \in C \mid \|x - q_0\| + \|x - q_g\| \le c \}\) be the informed set at cost bound \(c\)~\cite{gammell2018informed}, the ellipsoid with foci \(q_0,q_g\), meeting \(C_{\mathrm{free}}\) in the query ellipsoid \(\gE(q_0,q_g;c)\) of \cref{sec:problem}.
Stage \(\nu\) samples against the bound available on entry, \(c_{\nu-1}\), taking \(\pi_{\nu,m}\) of \cref{sec:problem} to be the uniform measure on \(\mathcal{X}_f(c_{\nu-1})\cap C_{\nu,m}\cap C_{\mathrm{free}}\), and the heuristic half of the mixture to be the midpoints of the closest layer-\((m{-}1)\) and layer-\((m{+}1)\) sample pairs.
Because the uniform grid shifts when \(M\) grows, AO-GTMP runs in \emph{epochs}: within an epoch \(M\) is
fixed and \(N_\nu\) increases until the per-layer hit probability saturates (the sample
schedule of \cref{thm:aogtmp}); a new epoch begins at larger \(M\) when triggered by the layer schedule:
\[
  M_\nu \;\ge\; \left\lceil \frac{c_\nu}{\Lambda_s(\tau)-2r-\varsigma}\right\rceil - 1,
\]
where the \emph{reach} \(\Lambda_s(\tau)=\sup\{\ell>0 \mid q_s(\ell)\ge1-\tau\}\) is the longest
edge the local planner solves with probability at least \(1-\tau\), for the success profile \(q_s\)
of \cref{ass:lp-profile}.
Cost convergence \(c_\nu \to c^*\) a.s.\ follows from \cref{thm:aogtmp}.

\section{Theoretical Analysis}
\label{sec:homotopy_theory}

We use the fact that a chain inside a tube of radius \(\delta\) around a reference path is homotopic to that reference.
Coverage then reduces to two questions: did every layer sample near the reference, and did the local planner connect consecutive samples.

Fix a reference path \(f:[0,1]\to C_{\mathrm{free}}\) from \(q_0\) to \(q_g\), rectifiable, parameterized at constant speed, of length \(L=\mathrm{TV}(f)\) and clearance \(\delta>0\), so \(N_\delta(\mathrm{Im}(f))\subset C_{\mathrm{free}}\).
For a radius \(r\in(0,\delta)\), the layer-\(m\) \emph{waypoint ball} \(A_m=B_r(f(t_m))\cap C_m\cap C_{\mathrm{free}}\), with \(B_r(f(t_m))\subset C_m\), collects the admissible layer-\(m\) waypoints; a single layer-\(m\) sample lands in it with probability at least \(p_m=\tfrac12\pi_m(A_m)\ge\tfrac12 a_m\lambda(B_r)\), the factor \(\tfrac12\) coming from the coin flip in the sampler (\cref{sec:problem}).
Consecutive waypoints are joined inside the tube by going in to the reference, along it, and back out, a path of length at most \(L/(M+1)+2r\) with clearance at least \(\delta-r\), so the local planner is called with the length limit:
\[
  \ell_M(r) \;=\; \frac{L}{M+1} + 2r + \varsigma
\]
for a fixed margin \(\varsigma>0\).
\begin{lemma}[\(\delta\)-tube homotopy]\label{lem:tube}
  Let \(w_0=q_0\), \(w_{M+1}=q_g\), \(w_m\in B_r(f(t_m))\), and let \(\Gamma_m\) join \(w_m\) to \(w_{m+1}\) inside \(\gE(w_m,w_{m+1};\ell_M(r))\).
  If \[
    \frac{3}{2}\frac{L}{M+1}+2r+\frac{\varsigma}{2}<\delta,
  \] then \(P=\Gamma_0*\cdots*\Gamma_M\) satisfies \(P\simeq f\) rel.\ endpoints.
  If every \(\Gamma_m\) is the straight segment \([w_m,w_{m+1}]\), the condition weakens to \(L/(M+1)+r<\delta\).
\end{lemma}

\begin{proof}
  Parameterize \(P\) at the reference speed, so that \(P|_{[t_m,t_{m+1}]}\) traverses \(\Gamma_m\), and fix \(t\in[t_m,t_{m+1}]\) with \(z=P(t)\).
  A point of \(\gE(w_m,w_{m+1};\ell_M(r))\) is within \(\ell_M(r)/2\) of one focus, each focus \(w_m\) lies within \(r\) of \(f(t_m)\), and constant speed keeps \(f(t)\) within \(L/(M+1)\) of both, so \[
    \|z-f(t)\|\le\frac{\ell_M(r)}{2}+r+\frac{L}{M+1}=\frac{3L}{2(M+1)}+2r+\frac{\varsigma}{2}<\delta.
  \]
  For the straight-line case, interpolating \([w_m,w_{m+1}]\) affinely against \([f(t_m),f(t_{m+1})]\) puts \(P(t)\) within \(r\) of that segment and the segment within \(L/(M+1)\) of \(f(t)\), so \[\|P(t)-f(t)\|\le r+\frac{L}{M+1}.\]
  In either case the homotopy \(H(t,\sigma)=(1-\sigma)f(t)+\sigma P(t)\) stays in \(B_\delta(f(t))\subset C_{\mathrm{free}}\) and fixes both endpoints.
\end{proof}

\begin{assumption}[Local planner success profile]\label{ass:lp-profile}
  Fix a clearance level \(\rho>0\).
  For each budget \(s\) there is a non-increasing \(q_s:(0,\infty)\to[0,1]\) with \(q_s(\ell)\to1\) as \(s\to\infty\) such that, whenever \(x,y\) admit a \(\rho\)-clear path of length at most \(\ell-\varsigma\), the call \(\mathsf{LP}(x,y;s,\ell)\) succeeds with probability at least \(q_s(\ell)\), conditionally on the sampling history and on the other calls.
\end{assumption}

The assumption asks only for success, since the length limit of \cref{sec:problem} already confines the path to the query ellipsoid; the clearance level, \(\rho=\delta-r\) below, is what keeps it meaningful, as no planner holds a uniform success rate over corridors of vanishing width.
RRT-Connect confined to that ellipsoid qualifies, with failure decaying as \(\exp(-\beta s\eta/\ell)\)~\cite{kleinbort2019probabilistic,karaman2011sampling} for the planner step \(\eta\) and a constant \(\beta>0\), and is \emph{short-range exact}: it tries the straight line first, so any collision-free straight-line edge is returned deterministically.

\subsection{Coverage from One Sampled Graph}
\label{sec:single-class}

\begin{theorem}[Coverage of one class]\label{thm:single-class}
  Under \cref{ass:lp-profile}, with \(r,M\) as in \cref{lem:tube} and every \(p_m>0\),
  \begin{multline*}
    \Pr\big[G^{\mathsf{LP}}_{M,N,s}\text{ covers }[f]\big] \\
    \ge\ 1 - \sum_{m=1}^M (1-p_m)^N - (M+1)\big(1-q_s(\ell_M(r))\big),
  \end{multline*}
  and the right-hand side tends to \(1\) as \(N,s\to\infty\) for every fixed \(M\).
\end{theorem}

\begin{proof}
  Layer \(m\) fails when none of its \(N\) samples lands in \(A_m\), which has probability at most \((1-p_m)^N\); a union bound over the \(M\) layers gives the first error term.
  Suppose instead every layer hits, and pick a waypoint \(w_m\in A_m\) in each.
  Consecutive waypoints are joined by \([w_m,f(t_m)]\), the reference subarc, and \([f(t_{m+1}),w_{m+1}]\): this path is collision-free, lies in the tube since \(r<\delta\), and has length at most \(L/(M+1)+2r=\ell_M(r)-\varsigma\), so it is \((\delta-r)\)-clear.
  Each of the \(M+1\) local planner calls therefore succeeds with probability at least \(q_s(\ell_M(r))\) by \cref{ass:lp-profile}, and a union bound over the edges gives the second term.
  When they all succeed, \cref{lem:tube} places the realized chain in the tube, hence in \([f]\).
\end{proof}

The same waypoint construction covers every \(\delta\)-clear class at once, since for a fixed clearance there are finitely many endpoint-fixed classes of bounded length.
The reach \(\Lambda_s(\tau)\) grows only sublinearly in the local budget, so past a point extra local effort adds little (\cref{fig:local_vs_global_tradeoff}).
A planner able to reach across the whole query \(\Lambda_s(\tau)>L+2r+\varsigma\) collapses the layered graph to a single call.

\subsection{Anytime Policies}
\label{sec:anytime-theory}

Coverage puts a chain of the target class in the graph; the search of \cref{sec:dp-search} then returns a minimum-cost chain in every realized class at once.

\begin{proposition}[Class-augmented value iteration is exact]\label{prop:class-dp}
  Assume the class label (i) \emph{composes}, \(\kappa(\Gamma*\Gamma')=\kappa(\Gamma)\kappa(\Gamma')\), as the h-signature does~\cite{bhattacharya2012topological}, and (ii) \emph{identifies classes}: two start-to-goal paths share a label if and only if they are homotopic.
  Let \(J_m(v,\chi)\) be the least cost of a chain from \(v\in V_m\) to \(q_g\) with label \(\chi\), and \(\gK\) the finite set of labels such chains carry.
  Then value iteration over the state \((v,\chi)\), \[
    J_m(v,\chi)=\min_{\{v',\,\chi'\,\mid\,\kappa(\Gamma_{v,v'})\chi'=\chi\}} c(v,v')+J_{m+1}(v',\chi'),
  \] from \(J_{M+1}(q_g,e)=0\) at the empty label \(e\) and \(+\infty\) at every other terminal label, computes \(J\) exactly, returning a minimum-cost chain in every class in \(\Theta(MN^2|\gK|)\) time.
\end{proposition}

\begin{proof}
  Costs add along a chain and, by (i), its label is the product of its edge labels, so any completion from layer \(m\) depends on the past only through \(v\) and the label \(\chi\) so far.
  Backward induction over this state is therefore exact, and by (ii) grouping chains by label groups them by class.
  Each layer relaxes its \(N^2\) edges against the \(|\gK|\) labels.
\end{proof}

\begin{theorem}[Anytime-GTMP class coverage]\label{thm:anytime-rr}
  Fix \((M,N,s)\) with per-stage coverage \(p_{\mathrm{cov}}=\prod_{m}(1-(1-p_m)^N)\,q_s(\ell_M(r))^{M+1}>0\).
  Then \([f]\) enters the archive almost surely, after at most \(1/p_{\mathrm{cov}}\) stages in expectation, and simultaneously for every \(\delta\)-clear class.
\end{theorem}

\begin{proof}
  Each stage draws fresh samples and independent local planner randomness, so the events \(A_\nu=\{\text{stage }\nu\text{ covers }[f]\}\) are independent, and \cref{thm:single-class} gives \(\Pr(A_\nu)\ge p_{\mathrm{cov}}\).
  As \(\sum_\nu\Pr(A_\nu)=\infty\), the Borel--Cantelli lemma makes \(A_\nu\) occur infinitely often almost surely, and \(\Pr(\bigcap_{\nu\le K}\neg A_\nu)\le(1-p_{\mathrm{cov}})^K\) bounds the first covering stage in expectation by \(1/p_{\mathrm{cov}}\).
  By \cref{prop:class-dp} each graph yields an archive entry, and the archive is monotone, so \([f]\) stays covered; with finitely many \(\delta\)-clear classes, intersecting these events covers them all at once.
\end{proof}

AO-GTMP instead fixes \(s\) and grows \((M_\nu,N_\nu)\uparrow\infty\), sampling each layer from the informed set \(\gX_f(c_{\nu-1})\)~\cite{gammell2018informed}, an instance of AO-x on the layered DAG~\cite{hauser2016asymptotically}; we state it for path length, but weighted lengths obey the same guarantee.

\begin{theorem}[AO-GTMP cost convergence]\label{thm:aogtmp}
  Assume a \(\delta\)-clear optimal path of length \(c^*\), a short-range-exact local planner, sampling regions containing \(\gX_f(c_0)\), and a sample schedule under which the number \(\Delta N_\nu\) of fresh samples at stage \(\nu\) satisfies \(\Delta N_\nu/(M_\nu^{d}\log M_\nu)\to\infty\).
  Then the best path length \(c_\nu\to c^*\) almost surely.
\end{theorem}

\begin{proof}
  Fix \(\varepsilon,\delta'>0\) and let \(\gamma_\varepsilon\) be a \(\delta'\)-clear path of length \(L_\varepsilon\le c^*+\varepsilon\)~\cite{kleinbort2020refined}, at constant speed.
  With \(M=M_\nu\) large enough that \(L_\varepsilon/(M{+}1)\le\delta'/2\), take waypoint balls: \[
    B_m=B_{r(M)}(\gamma_\varepsilon(t_m)),\quad r(M)=\min\left\{\frac{\delta'}{4},\ \frac{\varepsilon}{2(M{+}1)}\right\}.
  \]
  While \(c_{\nu-1}>c^*+2\varepsilon\), every \(y\in B_m\) has: \[
    \|y-q_0\|+\|y-q_g\|\le L_\varepsilon+2r(M)\le c^*+2\varepsilon<c_{\nu-1},
  \] and \(r(M)<\delta'\) keeps \(B_m\subset C_{\mathrm{free}}\), so \(B_m\subset\gX_f(c_{\nu-1})\cap C_{\mathrm{free}}\) is samplable in every layer.
  Each segment \([w_m,w_{m+1}]\) lies within \(r(M)+L_\varepsilon/\big(2(M{+}1)\big)<\delta'\) of \(\gamma_\varepsilon\), so short-range exactness realizes it deterministically---only ball-hitting is random---and the chain length is: \[
    \sum_m\|w_m-w_{m+1}\|\le L_\varepsilon+2(M{+}1)r(M)\le c^*+2\varepsilon.
  \]
  Each ball is hit by a single sample with probability: \[
    p_M\ge\frac{\lambda(B_{r(M)})}{2\lambda(\gX_f(c_0))}\ge c_\varepsilon M^{-d},
  \] so the chance some ball is unhit is at most \((M{+}1)(1-p_M)^{\Delta N_\nu}\to0\) under the schedule.
  Hence, with \(T=\bigcap_\nu\{c_\nu>c^*+2\varepsilon\}\), the conditional Borel--Cantelli lemma forces \(\Pr(T)=0\); intersecting over \(\varepsilon=1/k\) with \(c_\nu\ge c^*\) gives \(c_\nu\to c^*\) almost surely.
\end{proof}


\section{Experiments} \label{sec:experiments}
\begin{figure}[t]
  \centering
  \includegraphics[width=\linewidth]{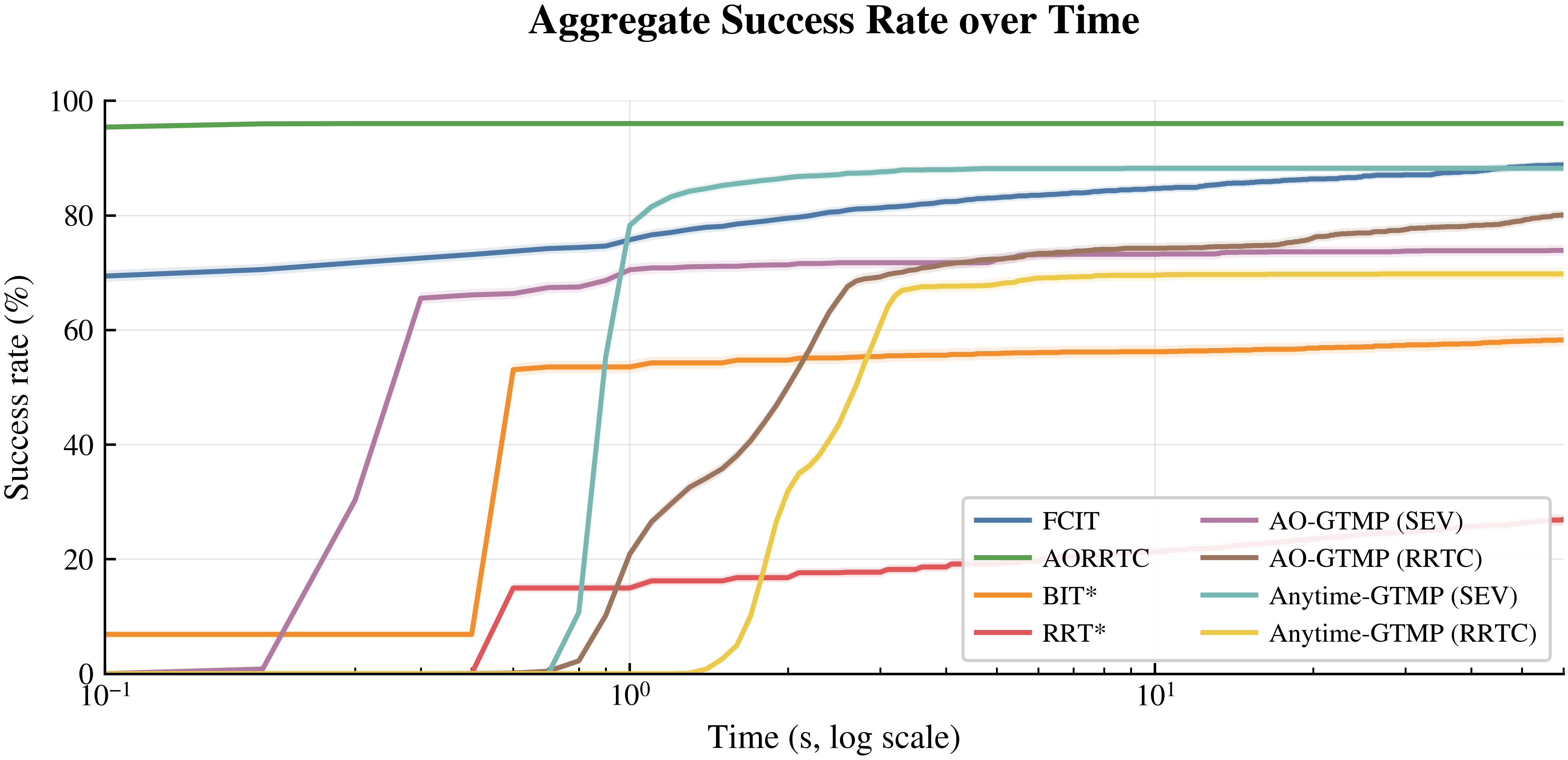}
  \caption{Success rate eCDFs on a 60 second budget, log scale. Anytime-GTMP SEV iterates fast and a low local planner budget leaves room for global exploration.}
  \label{fig:aggregate_success_rate_cdf}
\end{figure}

Experiments ran on an AMD Ryzen Threadripper PRO 5965WX 24-core CPU with 32\,GB of RAM and an NVIDIA RTX 4090.
We run each policy with two local planners: straight-line edges with collision checking (SEV) and RRT-Connect (RRTC).

\begin{table*}[t]
  \centering
  \small
  \setlength{\tabcolsep}{4pt}
  \resizebox{\textwidth}{!}{%
    \begin{tabular}{llccccccc}
      \toprule
      & Planner & \textbf{bookshelf small} & \textbf{bookshelf tall} & \textbf{bookshelf thin} & \textbf{table pick} & \textbf{table under pick} & \textbf{box} & \textbf{cage} \\
      \midrule
      \multirow{8}{*}{\rotatebox{90}{Panda}} & FCIT* & \textbf{100\%} (4.6) & \textbf{100\%} (4.6) & \textbf{100\%} (4.5) & 99\% (4.6) & \textbf{100\%} (5.7) & \textbf{100\%} (4.2) & 99\% (7.4) \\
      & AORRTC & \textbf{100\%} (\textbf{4.2}) & \textbf{100\%} (\textbf{4.4}) & \textbf{100\%} (\textbf{4.2}) & \textbf{100\%} (\textbf{4.4}) & \textbf{100\%} (4.7) & \textbf{100\%} (\textbf{3.9}) & \textbf{100\%} (\textbf{4.9}) \\
      & BIT* & 0\% (\(-\)) & 52\% (\(-\)) & 98\% (\(-\)) & 98\% (\(-\)) & \textbf{100\%} (5.6) & 1\% (\(-\)) & 20\% (\(-\)) \\
      & RRT* & 32\% (\(-\)) & 29\% (\(-\)) & 34\% (\(-\)) & 48\% (\(-\)) & 37\% (\(-\)) & 24\% (\(-\)) & 0\% (\(-\)) \\
      & AO-GTMP (SEV) & \textbf{100\%} (4.4) & \textbf{100\%} (4.5) & \textbf{100\%} (4.4) & 99\% (4.5) & \textbf{100\%} (5.8) & \textbf{100\%} (4.2) & \textbf{100\%} (6.8) \\
      & AO-GTMP (RRTC) & 92\% (4.3) & 97\% (\textbf{4.4}) & \textbf{100\%} (\textbf{4.2}) & \textbf{100\%} (\textbf{4.4}) & \textbf{100\%} (\textbf{3.7}) & \textbf{100\%} (\textbf{3.9}) & 96\% (6.3) \\
      & Anytime-GTMP (SEV) & \textbf{100\%} (4.7) & \textbf{100\%} (4.8) & \textbf{100\%} (4.8) & 99\% (4.8) & \textbf{100\%} (6.6) & \textbf{100\%} (4.8) & \textbf{100\%} (7.0) \\
      & Anytime-GTMP (RRTC) & 91\% (4.3) & 98\% (4.5) & 98\% (4.4) & \textbf{100\%} (\textbf{4.4}) & \textbf{100\%} (\textbf{3.7}) & \textbf{100\%} (\textbf{3.9}) & 88\% (6.8) \\
      \midrule
      \multirow{8}{*}{\rotatebox{90}{Ur5}} & FCIT* & 94\% (6.3) & 95\% (6.5) & \textbf{99\%} (6.2) & \textbf{100\%} (5.9) & \textbf{100\%} (6.2) & \textbf{100\%} (4.6) & \textbf{100\%} (7.0) \\
      & AORRTC & \textbf{96\%} (6.0) & \textbf{98\%} (\textbf{6.0}) & \textbf{99\%} (5.7) & \textbf{100\%} (5.6) & \textbf{100\%} (5.5) & \textbf{100\%} (\textbf{4.2}) & \textbf{100\%} (\textbf{5.6}) \\
      & BIT* & \textbf{96\%} (6.4) & 95\% (6.6) & \textbf{99\%} (6.3) & \textbf{100\%} (5.9) & 99\% (\(-\)) & \textbf{100\%} (4.7) & 48\% (\(-\)) \\
      & RRT* & 61\% (\(-\)) & 44\% (\(-\)) & 46\% (\(-\)) & 45\% (\(-\)) & 63\% (\(-\)) & 67\% (\(-\)) & 0\% (\(-\)) \\
      & AO-GTMP (SEV) & \textbf{96\%} (6.1) & 95\% (6.2) & \textbf{99\%} (5.9) & \textbf{100\%} (5.6) & \textbf{100\%} (6.4) & \textbf{100\%} (4.4) & \textbf{100\%} (7.0) \\
      & AO-GTMP (RRTC) & \textbf{96\%} (\textbf{5.9}) & 96\% (\textbf{6.0}) & \textbf{99\%} (\textbf{5.6}) & \textbf{100\%} (\textbf{5.5}) & 99\% (\textbf{4.8}) & \textbf{100\%} (\textbf{4.2}) & 83\% (\textbf{5.8}) \\
      & Anytime-GTMP (SEV) & \textbf{96\%} (6.3) & 95\% (6.6) & \textbf{99\%} (6.3) & \textbf{100\%} (6.0) & \textbf{100\%} (7.3) & \textbf{100\%} (5.0) & \textbf{100\%} (7.2) \\
      & Anytime-GTMP (RRTC) & 94\% (6.0) & 95\% (6.3) & 98\% (5.8) & \textbf{100\%} (\textbf{5.6}) & \textbf{100\%} (\textbf{4.8}) & \textbf{100\%} (\textbf{4.2}) & 86\% (7.7) \\
      \midrule
      \multirow{8}{*}{\rotatebox{90}{Fetch}} & FCIT* & 43\% (9.3) & 47\% (8.8) & 25\% (9.5) & 96\% (6.6) & \textbf{100\%} (5.6) & 98\% (7.7) & 88\% (11.0) \\
      & AORRTC & \textbf{97\%} (7.2) & \textbf{97\%} (\textbf{6.8}) & \textbf{94\%} (\textbf{5.7}) & \textbf{100\%} (5.4) & \textbf{100\%} (3.7) & \textbf{100\%} (6.4) & \textbf{100\%} (\textbf{6.2}) \\
      & BIT* & 5\% (\(-\)) & 6\% (\(-\)) & 3\% (\(-\)) & 71\% (\(-\)) & 70\% (\(-\)) & 61\% (\(-\)) & 14\% (\(-\)) \\
      & RRT* & 0\% (\(-\)) & 1\% (\(-\)) & 0\% (\(-\)) & 24\% (\(-\)) & 3\% (\(-\)) & 10\% (\(-\)) & 0\% (\(-\)) \\
      & AO-GTMP (SEV) & 53\% (11.3) & 40\% (13.0) & 66\% (16.9) & \textbf{100\%} (7.7) & \textbf{100\%} (6.1) & 99\% (10.4) & 64\% (18.7) \\
      & AO-GTMP (RRTC) & 14\% (\textbf{5.9}) & 19\% (\textbf{6.1}) & 2\% (9.4) & 91\% (\textbf{4.9}) & 95\% (\textbf{2.9}) & 99\% (\textbf{5.6}) & 27\% (12.9) \\
      & Anytime-GTMP (SEV) & 81\% (9.0) & 74\% (8.7) & 84\% (9.7) & \textbf{100\%} (7.7) & \textbf{100\%} (6.9) & 99\% (9.0) & 88\% (13.0) \\
      & Anytime-GTMP (RRTC) & 15\% (9.2) & 20\% (8.5) & 4\% (8.3) & 82\% (5.3) & 89\% (3.7) & 97\% (\textbf{5.6}) & 16\% (9.4) \\
      \bottomrule
    \end{tabular}%
  }
  \caption{\textbf{Planner performance on the MotionBenchMaker (MBM) suite.} Each planner is evaluated on 7 problems for the Panda, UR5, and Fetch robots. For every (planner, problem) pair, success rate and mean simplified path length (in paranthesis) is reported. As planners differ in which trials they solve, path cost is aggregated only over trials solved by \emph{every} reference planner (FCIT, AORRTC, and GTMP) in that column, so reported costs compare planners on a common set of problem instances. A dash (\(-\)) in the cost position indicates that the planner did not solve that whole common set, leaving no comparable cost. Each rate is over 100 trials per (planner, problem) pair. Within each robot block the highest success rate in a column is shown in \textbf{bold}. For path cost we \textbf{bold} the lowest mean cost together with every planner. On cost, AO-GTMP is bold on 5/7 Panda, 7/7 UR5, and 5/7 Fetch problems, and Anytime-GTMP on 4/7, 3/7, and 1/7, against 6/7, 3/7, and 3/7 for the strongest baseline, AORRTC.}
  \label{tab:mbm_sweep}
\end{table*}

\subsection{Feasibility and Cost}

We validate feasibility and cost on MotionBenchMaker~\cite{chamzas2022-motion-bench-maker} (\cref{tab:mbm_sweep,fig:aggregate_success_rate_cdf}).
GTMP variants are not the fastest to a first solution: within one second AORRTC and FCIT succeed often, BIT* stays below \(10\%\), and GTMP variants rarely return, which we attribute to the overhead of layered graph construction and local planner evaluation.
Over the full time budget, Anytime-GTMP (SEV) matches FCIT's final success rate of approximately \(85\%\).
On the common set solved by every reference planner, the variants match or beat the lowest mean path cost most often: AO-GTMP on \(5/7\) Panda, \(7/7\) UR5, and \(5/7\) Fetch problems, and Anytime-GTMP on \(4/7\), \(3/7\), and \(1/7\), against \(6/7\), \(3/7\), and \(3/7\) for AORRTC.

\Cref{sec:single-class} predicts that a strong enough local planner collapses the graph to a single call: fixing \(M=6\) layers, \(N=100\) samples per layer, and an AORRTC budget of 1000 iterations, a complete start-to-goal path emerges near 250 iterations and approaches the optimal chain by 1000, as the graph approaches the connectivity of the continuous configuration space.

\begin{figure*}[t]
  \centering
  \includegraphics[width=0.47\textwidth]{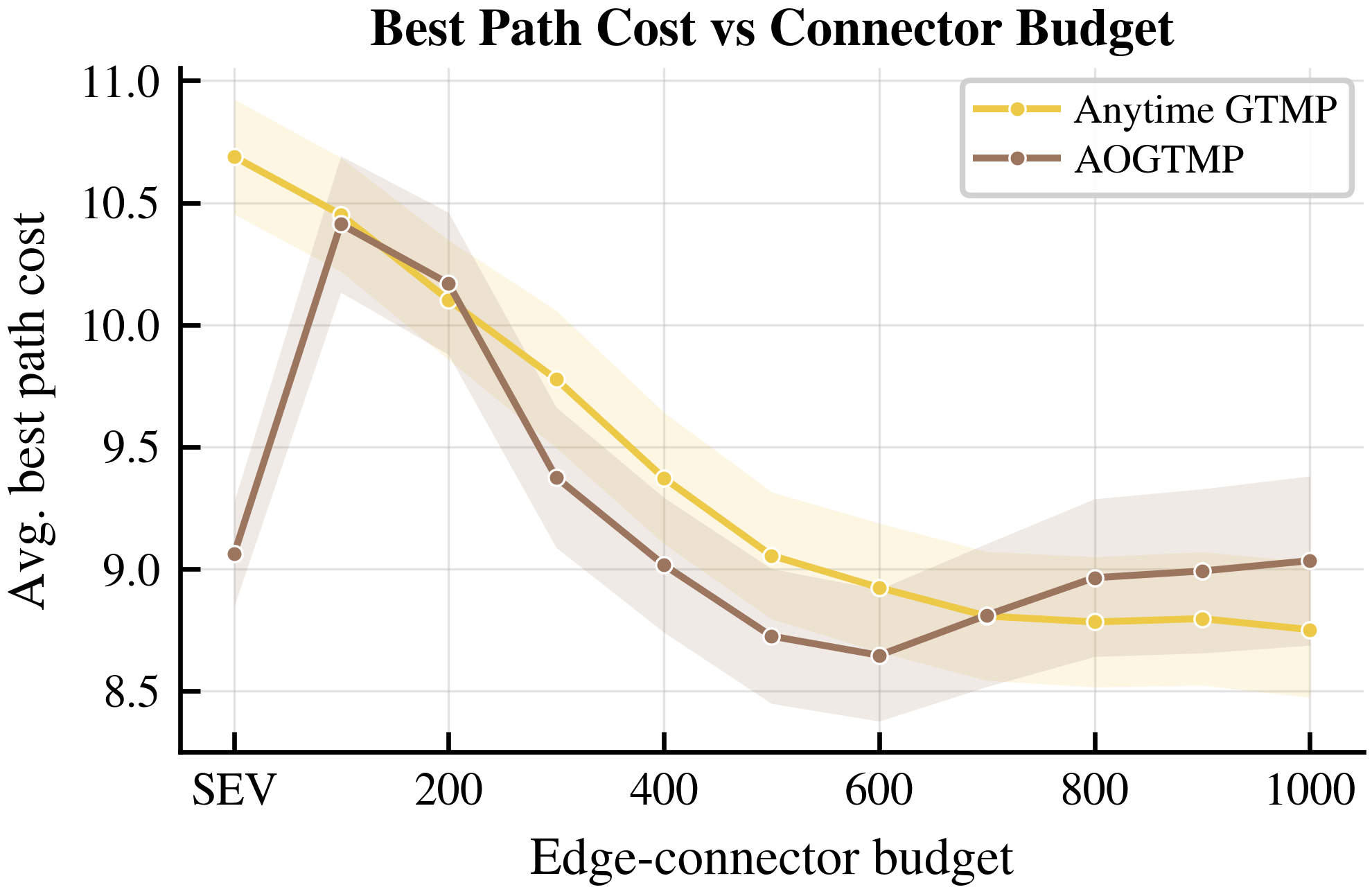}
  \hfill
  \includegraphics[width=0.47\textwidth]{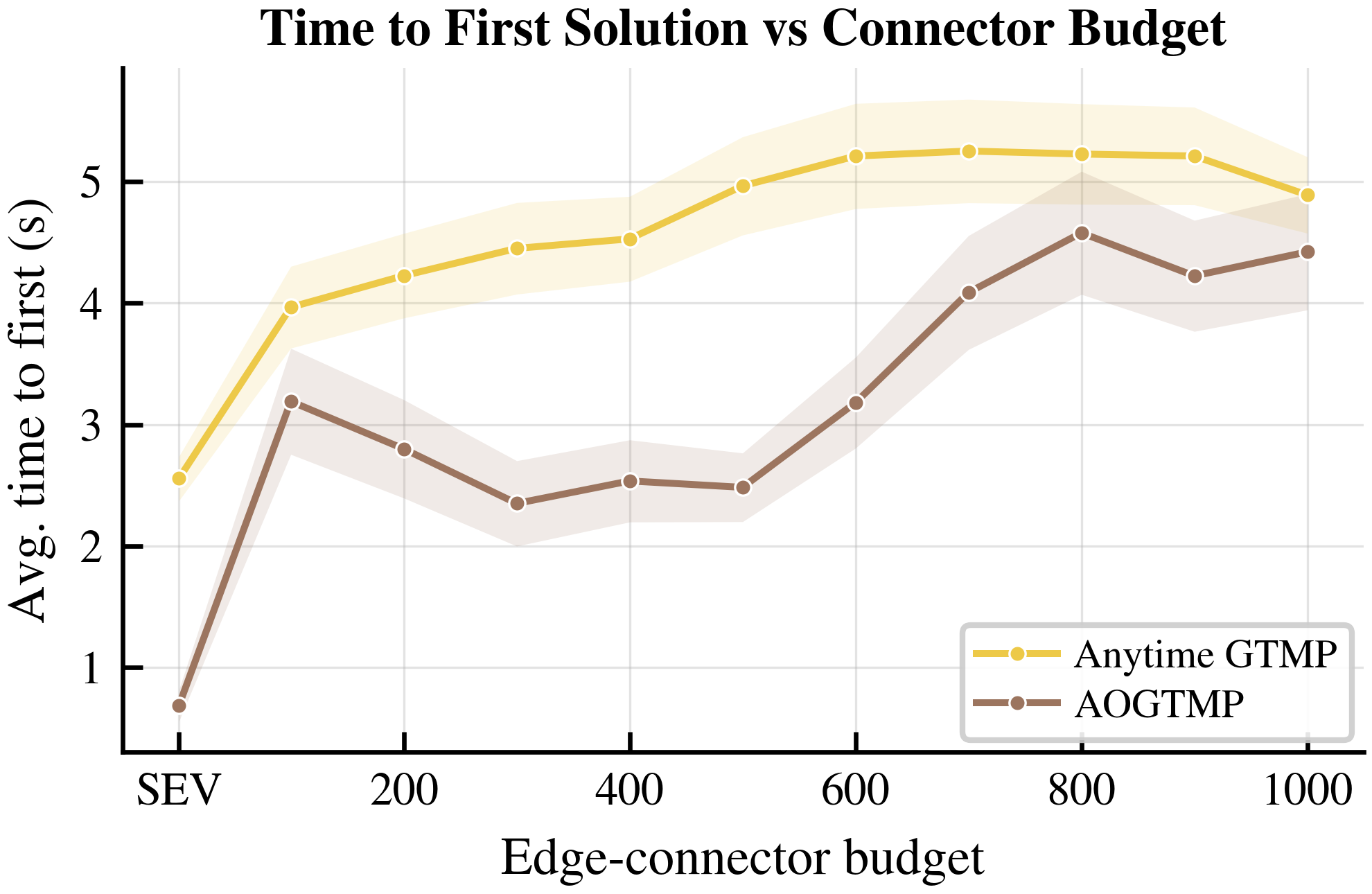}
  \caption{Per-edge local planner budget against (a) best path cost and (b) time to first solution, for Anytime-GTMP (yellow) and AO-GTMP (brown). Budget 1 is a straight-line local planner; 100--1000 are maximum RRT-Connect iterations per edge. Bands are means over solved problems with SEM.}
  \label{fig:local_vs_global_tradeoff}
\end{figure*}

Sweeping the per-edge budget in \cref{fig:local_vs_global_tradeoff} shows diminishing returns: path cost improves until roughly 400 to 600 RRT-Connect iterations and then flattens, so under a fixed planning budget moderate local effort with more global sampling dominates heavy local effort. 

\subsection{Topological Diversity}

\begin{figure*}[t]
  \centering
  \includegraphics[width=0.48\textwidth]{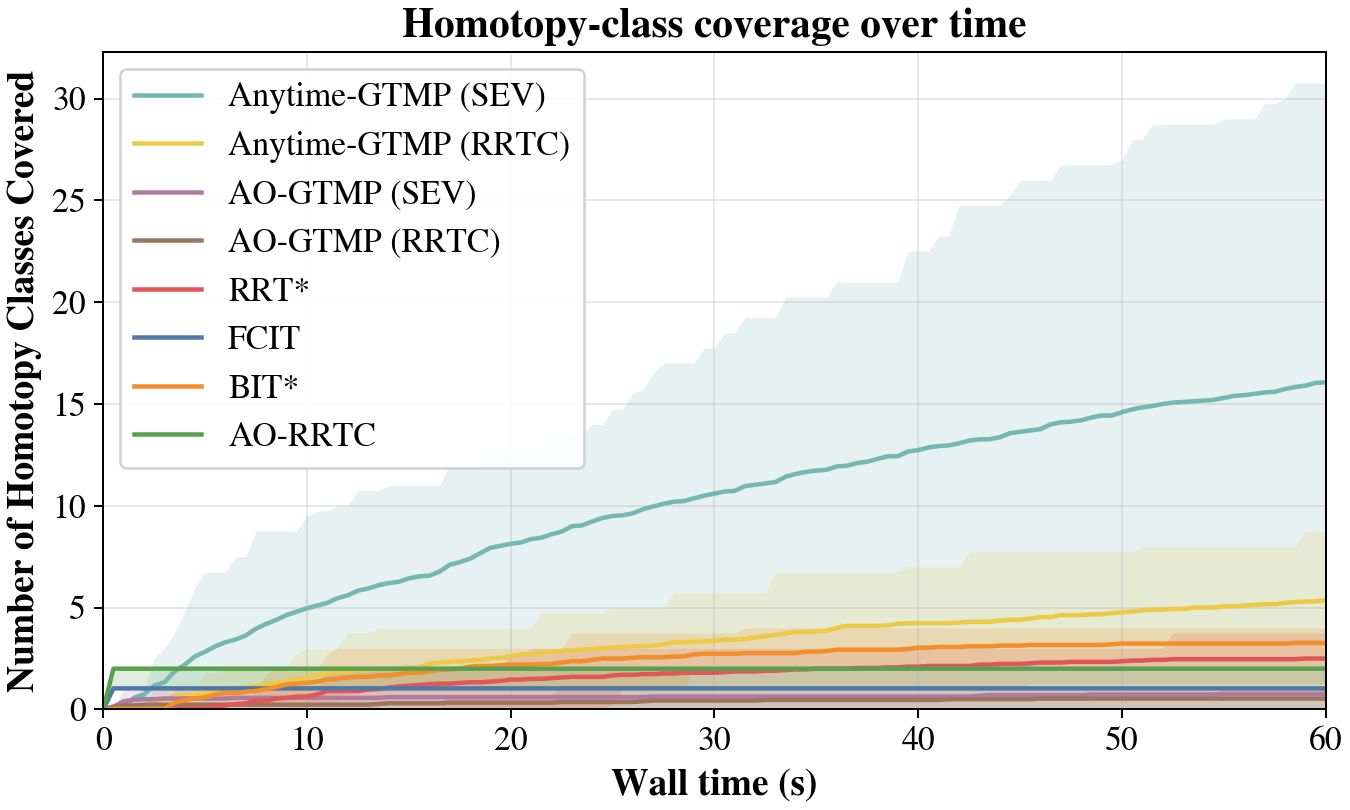}
  \hfill
  \includegraphics[width=0.48\textwidth]{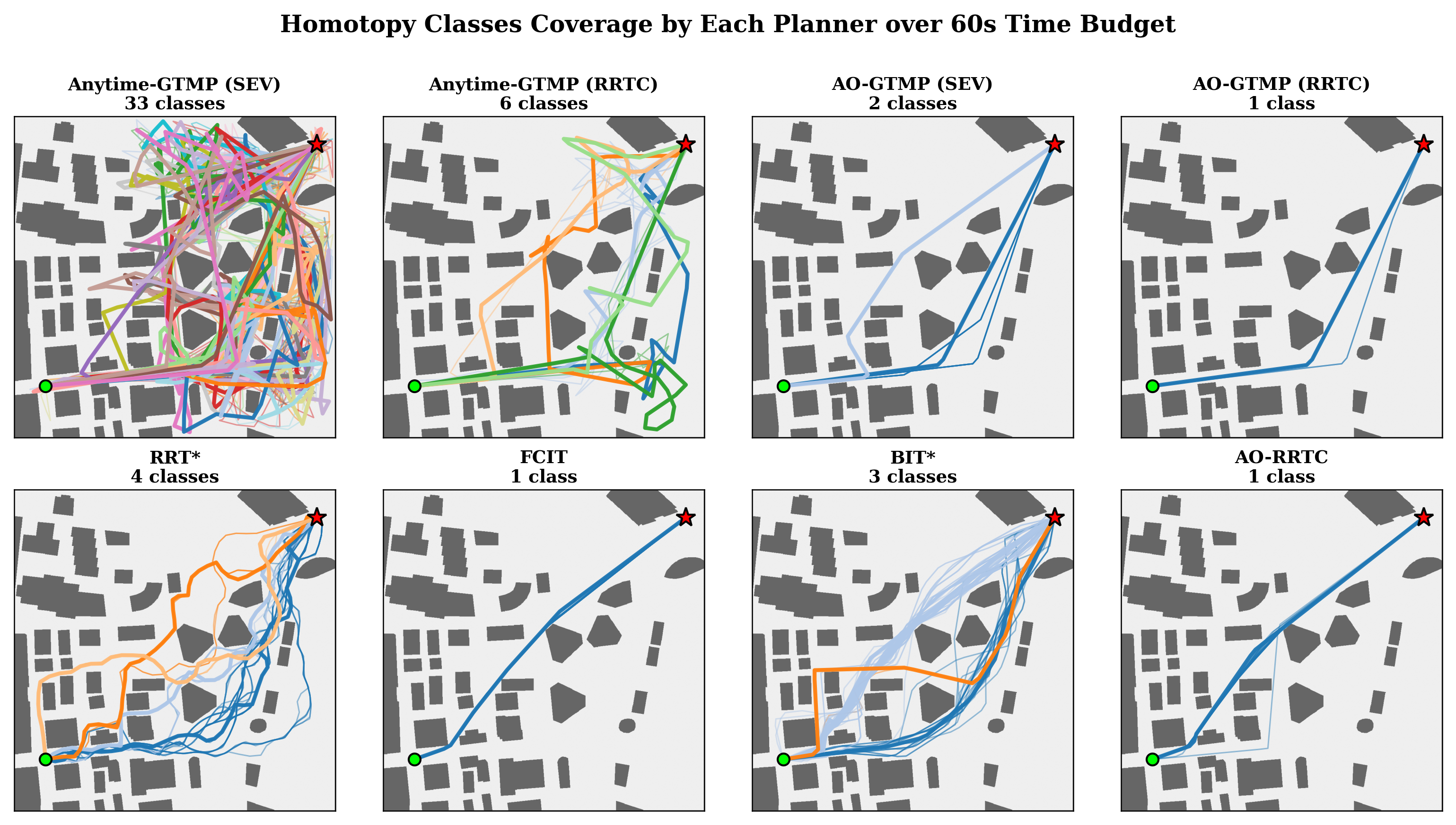}
  \caption{Homotopy-class coverage across the 2D maps: aggregate coverage (left) and per-planner classes on the Sydney map (right). Random restart keeps sampling new corridors, so Anytime-GTMP reaches the most distinct classes, whereas AO-GTMP concentrates samples in the shrinking informed set.}
  \label{fig:path_diversity}
\end{figure*}

We evaluate topological diversity (\cref{fig:diversity_vs_cost_bound,fig:path_diversity}) on 2D street-view heightmaps from Sturtevant's database~\cite{sturtevant2012benchmarks}, whose homotopy classes are visually verifiable; the baselines return one path per query and do not target class coverage.
Over a 60\,s budget we record a \textit{path event} whenever a planner explores a start-to-goal path, returned or not, label each event with a homotopy invariant, and pre-cluster with Dynamic Time Warping, itself not a homotopy invariant.
Anytime-GTMP covers the highest average number of classes, while aggressive cost bounding concentrates informed planners on a few near-optimal ones: on the Sydney map, AO-GTMP (RRTC) and FCIT each identify a single class.

\begin{figure}[t]
  \centering
  \includegraphics[width=\linewidth]{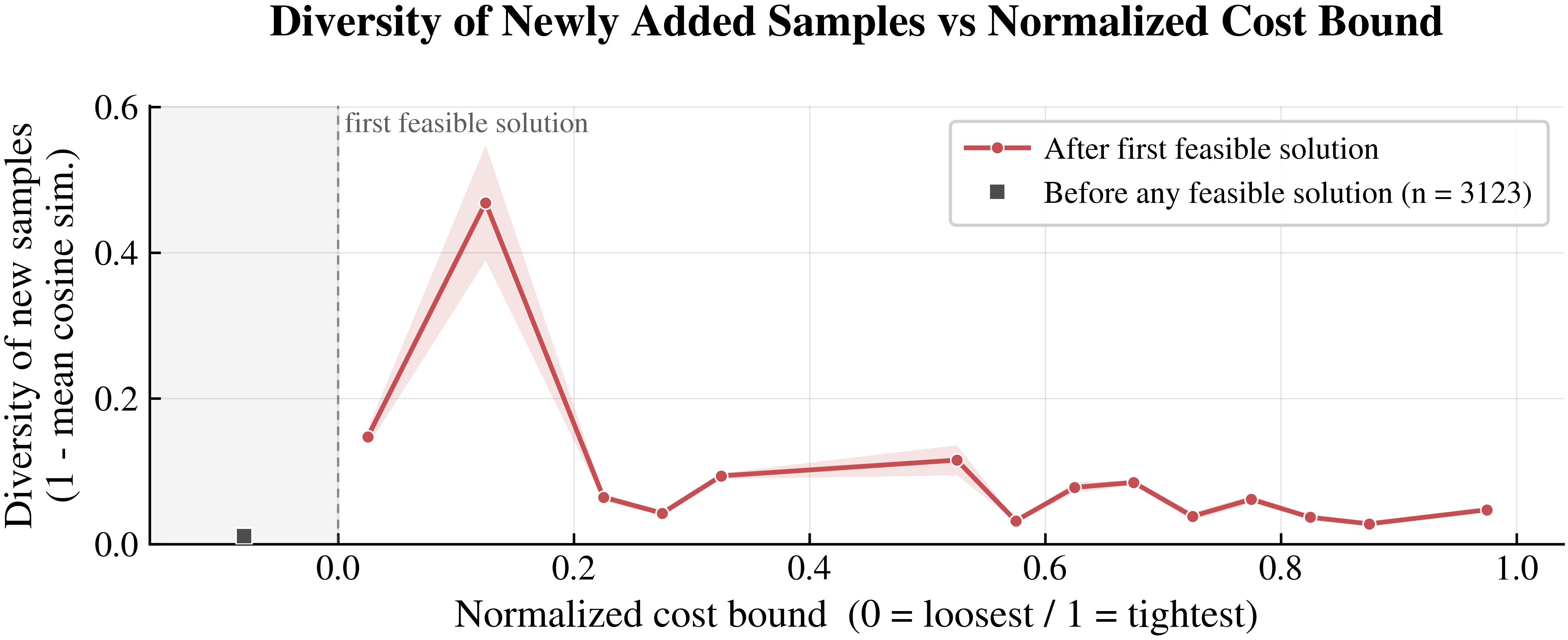}
  \caption{As AO-GTMP's cost bound shrinks, so does the sampling volume, and the diversity of new samples falls as search narrows to near optimal.}
  \label{fig:diversity_vs_cost_bound}
\end{figure}

\Cref{fig:diversity_vs_cost_bound} plots the diversity of new AO-GTMP samples, the complement of mean cosine similarity, against the normalized cost bound: it spikes to \(0.6\) just after the first feasible solution, when the informed set is still large, then falls as \(c_\nu\) tightens and \(\gX_f(c_\nu)\) contracts.
The trade-off is empirical as well as theoretical: repetition gives diversity, and a shrinking informed set gives optimality.
\section{Conclusion} \label{sec:conclusion}

We have generalized GTMP~\cite{le2025global} so that adjacent-layer edges are realized by any black-box local planner.
We prove that one sampled graph covers each \(\delta\)-clear homotopy class with high probability (\cref{thm:single-class}), and every such class of bounded length simultaneously; that the longest edge a local planner connects reliably grows only sublinearly in its budget, while each additional sample per layer cuts the miss probability exponentially; and that two iteration policies yield distinct guarantees: random restart at a fixed budget achieves almost-sure class coverage (\cref{thm:anytime-rr}), while informed expansion with growing budgets drives cost to the optimum (\cref{thm:aogtmp}).
Empirically, the variants match FCIT's success rate at competitive path cost on 6--8 DoF manipulation, and on 2D navigation Anytime-GTMP covers the most homotopy classes while AO-GTMP concentrates on near-optimal ones---the random-restart versus informed-expansion trade-off in practice.

Future work includes lazy variants that realize only a subset of adjacent-layer pairs, scheduling that adapts \((M_\nu,N_\nu,s_\nu)\) online from the measured local planner profile, dynamically-aware local planners that extend the same guarantees to dynamically feasible trajectories, and evaluating topological coverage on high-DoF manipulation, where homotopy classes are harder to verify than in the 2D navigation studied here.

\bibliographystyle{IEEEtran}
\bibliography{bib/IEEEabrv,bib/references_sub}

\end{document}